\documentclass[runningheads]{llncs}
\usepackage{graphicx}
\usepackage{hyperref} 
\usepackage{amssymb,amsmath,amsfonts,epsf,epsfig,graphicx,mathrsfs,tabularx,color}

\newcommand{\bi}{\begin{itemize}}
\newcommand{\ei}{\end{itemize}}
\newcommand{\be}{\begin{enumerate}}
\newcommand{\ee}{\end{enumerate}}

\begin{document}
\title{Generalising realisability in statistical learning theory under epistemic uncertainty} 
%
%
\author{Fabio Cuzzolin\inst{1}} 
\authorrunning{F. Cuzzolin}
%
\institute{Oxford Brookes University, UK \\ \email{fabio.cuzzolin@brookes.ac.uk}
}
\maketitle              
\begin{abstract}
The purpose of this paper is to look into how central notions in statistical learning theory, such as realisability, generalise under the assumption that train and test distribution are issued from the same \emph{credal set}, i.e., a convex set of probability distributions. This can be considered as a first step towards a more general treatment of statistical learning under epistemic uncertainty.

\end{abstract}

\section{Introduction} \label{sec:introduction}

Statistical learning theory \cite{vapnik1999overview,vapnik2013nature} considers the problem of predicting an output $y \in \mathcal{Y}$ given an input $x \in \mathcal{X}$, by means of a mapping $h : \mathcal{X} \rightarrow \mathcal{Y}$, $h \in \mathcal{H}$ called \emph{model} or \emph{hypothesis}, which lives in a specific model space $\mathcal{H}$. The error committed by a model is measured by a \emph{loss function} $l:(\mathcal{X} \times \mathcal{Y}) \times \mathcal{H} \rightarrow \mathbb{R}$, for instance the zero-one loss $l((x, y), h) = \mathbb{I}[y \neq h(x)]$.
The input-output pairs are assumed to be generated a probability distribution $p^*$. 

The \emph{expected risk} of a model $h$
\begin{equation} \label{eq:expected-risk}
L(h) \doteq \mathbb{E}_{(x,y)\sim p^*} [l((x,y), h)]
\end{equation}
is measured as its expected loss $l$ assuming that the pairs $(x_1,y_1),(x_2,y_2),...$ are sampled i.i.d. from the probability distribution $p^*$. The  \emph{expected risk minimiser}
\begin{equation} \label{eq:expected-risk-minimiser}
h^* \doteq \arg \min_{h \in \mathcal{H}} L(h),
\end{equation}
is any hypothesis in the given model space $\mathcal{H}$ that minimises the expected risk (not a random quantity at all). 
Given $n$ training examples, also drawn i.i.d. from $p^*$, the \emph{empirical risk} of a hypothesis $h$ is the average loss over the available training set $\mathcal{D} = \{(x_1,y_1), ..., (x_n,y_n) \}$:
\begin{equation}
\hat{L}(h) \doteq \frac{1}{n} \sum_{i=1}^n l((x_i,y_i),h).
\end{equation}
We can then define the \emph{empirical risk minimiser} (ERM) as:
\begin{equation} \label{eq:empirical-risk-minimiser}
\hat{h} \doteq \arg \min_{h \in \mathcal{H}} \hat{L}(h).
\end{equation}
This is, instead, a random variable which depends on the training data.

In alternative, we are interested in estimating how well the ERM does with respect to the best theoretical model $h^* \in \mathcal{H}$ in the given class:
\begin{equation}
    \label{eq:excess-risk}
L(\hat{h}) - L(h^*)
\end{equation}
This is called the \emph{excess risk}, and is a random variable depending on the training set $T$ (through $\hat{h}$).

Statistical learning theory seeks upper bounds on the difference between the expected risk and the empirical risk of the ERM (which is the only thing we can compute from the available training set), under increasingly more relaxed assumptions about the nature of the hypotheses space $\mathcal{H}$. Two common such assumptions are that either the model space is finite, or that there exists a model with zero expected risk (\emph{realisability} \cite{meir1995stochastic}).

In real-world situations, however, the data distribution may (and often does) vary, causing issues of \emph{domain adaptation} \cite{farahani2021brief} (DA) or \emph{generalisation} \cite{zhou2022domain} (DG).
Domain adaptation and generalisation are interrelated yet distinct concepts in machine learning, as they both deal with the challenges of transferring knowledge across different domains. The main goal of DA is to adapt a machine learning model trained on source domains to perform well on target domains \cite{you2019universal}. In opposition, DG aims to train a model that can generalise well to unseen data/domains not available during training \cite{Piva_2023_WACV}.

Some attempts to derive generalization bounds under more realistic conditions
within classical statistical learning theory have been made.
However, those approaches are characterized by a lack of generalisability, and the use of strong assumptions \cite{caprio2024credal}.
\emph{Imprecise probabilities} \cite{shafer1976mathematical,walley2000towards,cuzzolin2018visions,cuzzolin2021springer}, on the other hand, can provide a radically different solution to the construction of bounds in learning theory. A hierarchy of formalisms aimed at mathematically modeling the ``epistemic'' uncertainty induced by sources such as lack of data, missing data or data which is imprecise in nature \cite{hullermeier2021aleatoric,second-order,volume_caprio}, e.g. in the form of upper and lower bounds on probabilities \cite{cuzzolin2003geometry}, belief functions \cite{cuzzolin2001geometric,cuzzolin2013l_} or convex sets of distributions \cite{antonucci2010credal}, imprecise probabilities have been successfully employed in the design of neural networks providing both better accuracy and uncertainty quantification to predictions \cite{sensoy2018,manchingal2022,manchingal2023,caprio_IBNN,ibcl,manchingal2024}.

To date, however, they have never been considered as a tool to address the foundational issues of statistical learning theory associated with data drifting. The purpose of this paper is to look into how central notions in statistical learning theory, such as realisability, generalise under the assumption that train and test distribution are issued from the same \emph{credal set}, i.e., a convex set of probability distributions.

\section{Probably Approximately Correct (PAC) learning}

The core notion of classical statistical learning theory is that of \emph{probably approximately correct} algorithms. 
\begin{definition}
A learning algorithm is \emph{probably approximately correct} (PAC) if it finds with probability at least $1-\delta$ a model $h \in \mathcal{H}$ which is `approximately correct', i.e., it makes a training error of no more than $\epsilon$.
\end{definition}
According to this definition, PAC learning aims at providing bounds of the kind:
\[
P[L(\hat{h}) - L(h^*) > \epsilon] \leq \delta,
\]
on the difference between the loss of the ERM and the minimal theoretical loss for that class of models. 
Note that the randomness of this event ($P$) is due to the random selection of a particular training set $T$ from a hypothetical collection $\mathcal{T}$.

This is to account for the \emph{generalisation problem} \cite{arjovsky2021out}, the fact that the error we commit when training a model is different from the error one can expect on entirely new data.
The main issue with traditional statistical learning theory is the assumption that training and test data are sampled from the same (unknown) probability distribution. Machine learning deployment `in the wild' \cite{krueger2021out} has shown that that is hardly the case, leading to sometimes catastrophic failure\footnote{{\url{https://www.nytimes.com/2018/03/19/technology/uber-driverless-fatality.html}}}.

In this paper, we take a first step towards robustifying PAC learning, by analysing the proof of its generalisation bounds in the case of \emph{finite, realisable} models (Section \ref{sec:proof}), and sketching a proof under the assumption that training and test distributions come from a same \emph{credal set} \cite{kyburg1987bayesian,levi1980enterprise} (convex set of distributions), Section \ref{sec:sketch}. 

Firstly, however, we review a number of major results from classical statistical learning theory.

\section{Results from statistical learning theory} \label{sec:slt}

\subsection{Bounds for finite, realisable case}

When (1) the model space $\mathcal{H}$ is finite, and (2) there exists a hypothesis $h^* \in \mathcal{H}$ that obtains zero expected risk, that is:
\begin{equation} \label{eq:realisability}
L(h^*) = \mathbb{E}_{(x,y)\sim p^*} [l((x,y), h^*)] = 0,
\end{equation}
a property called \emph{realisability}, PAC holds with

\[ \epsilon = \frac{\log |\mathcal{H}| + \log (1/\delta)}{n}. \]
As it can be verified in\footnote{\url{https://web.stanford.edu/class/cs229t/notes.pdf}}, crucial to the proof of this initial result is the following \emph{union bound} from classical probability theory, which obviously only holds for finite collections of sets:
\begin{equation} \label{eq:union-bound} 
P(A_1 \cup ... \cup A_K) \leq \sum_{k=1}^K P(A_k).
\end{equation}

How can such a result be extended to the case in which realisability does not hold, or we deal with an infinite model space?
Fortunately, the PAC bound can be reduced to a statement about \emph{uniform convergence}, as follows:

\begin{equation} \label{eq:uniform} 
P[L(\hat{h}) - L(h^*) \geq \epsilon] \leq P \left [ \sup_{h \in \mathcal{H}} | L(h) - \hat{L}(h)| \geq \frac{\epsilon}{2} \right ]. 
\end{equation}


Such proofs leverage standard \emph{concentration inequalities} from probability theory, in particular:

\begin{enumerate}
    \item the \emph{Markov inequality} \cite{shadrin2004twelve}
    \[ P[Z \geq t] \leq \frac{E[Z]}{t}; \]
    \item and the \emph{Hoeffding inequality} \cite{bentkus2004hoeffding}
    \begin{equation}\label{eq:Hoeffding} 
    P [ \hat{\mu}_n \geq E[\hat{\mu}_n] + \epsilon ] \leq 
    \exp{\left ( \frac{-2n^2 \epsilon^2}{\sum_{i=1}^n (b_i - c_i)^2} \right )},
    \end{equation}
    where $a_i \leq x_i \leq b_i$ are i.i.d random variables, and $\hat{\mu}_n = \frac{1}{n} \sum_i x_i$. 
\end{enumerate}

\subsection{PAC bounds for finite model spaces}

By using uniform convergence (\ref{eq:uniform}) in conjunction with the Hoeffding inequality (\ref{eq:Hoeffding}) on can get a PAC bound for the finite case without realisability:
\[
\epsilon = \sqrt{\frac{2 (\log |\mathcal{H}| + \log (2/\delta))}{n}}
\]
(see \cite{liang}, equation (183)).

\subsection{PAC bounds for infinite model spaces}

In machine learning, however, in most cases, the model space is infinite (think for instance of support vector machines \cite{pisner2020support}, KNN classifiers, convolutional neural networks \cite{li2021survey}, just to cite a few model types).
Thus, to be useful, generalisation bounds need to apply to infinite model spaces too.
However, when $\mathcal{H}$ is not finite, we cannot use the union-bound constraint (\ref{eq:union-bound}) anymore.

Fortunately, one can use \emph{Mc Diarmid's inequality} \cite{combes2015extension} instead (\cite{liang}, Theorem 8):
\begin{equation} \label{eq:diarmids} 
P [ f(X_1,\ldots, X_n) - E[f(X_1,\ldots, X_n)] \geq \epsilon ] \leq \exp{\left ( \frac{-2\epsilon^2}{\sum_{i=1}^n} c_i^2 \right )}
\end{equation}
whenever the \emph{bounded difference} condition is satified:

\begin{equation} \label{eq:bounded-difference}
| f(x_1,\ldots,x_i,\ldots, x_n) - f(x_1,\ldots,x'_i,\ldots, x_n) | \leq c_i \quad \forall i, x
\end{equation}
for $X_1, \ldots, X_n$ independent random variables. Note that the proof of this inequality requires the notion of \emph{martingale} from probability theory \cite{doob1971martingale}.

In fact, Mc Diarmid's inequality generalises Hoeffding's inequality.

\subsubsection{Derivation}

The derivation of generalisation bounds for infinite model spaces follows this series of steps:
\begin{itemize}
    \item 
The uniform convergence bound (\ref{eq:uniform}) can be expressed in terms of the following variable:
\[
G_n \doteq \sup_{h \in \mathcal{H}} L(h) - \hat{L}(h).
\]
\item
Note that $G_n$ is a deterministic function of the training set $T \in \mathcal{T}$, and satisfies the bounded difference condition (\ref{eq:bounded-difference}).
\item
We can then apply Mc Diarmid's (\ref{eq:diarmids}), and get the \emph{tail bound}:
\begin{equation}
\label{eq:tail-bound}
P [G_n \geq E[G_n] + \epsilon ] \leq \exp{- 2 n \epsilon^2}.
\end{equation}
\item
Now, to obtain the desired PAC bounds for the infinite case, we need to bound the quantity $E[G_n]$.
\item
This, however, depends on the (unknown) data distribution $p^*$.
\end{itemize}

\subsubsection{Symmetrisation and notion of ``ghost'' dataset}

A technique called \emph{symmetrisation} \cite{bousquet2003introduction} can remove the dependency from $p^*$ (which is unknown) to a dependency from the training set $T$ (which we have).

This consists in introducing a \emph{ghost dataset} $T ' = \{ X'_1, \ldots, X'_n \}$ drawn i.i.d. from $p^*$, so that (see \cite{liang}, equation (210)):
\[
E[G_n] = E \left [ \sup_h E[\hat{L}'(h)] - \hat{L}(h) \right ].
\]
This allows us to get a bound for $E[G_n]$ which depends on $p^*$ only via the original dataset $T$ and the ghost dataset $T'$:
\[
E[G_n] \leq E \left[ 
\sup_{h \in H} \frac{1}{n} \sum_{i=1}^n \big [ l(X'_i,h) - l(X_i, h) \big ]
\right ].
\]
We can then remove the dependency from the ghost dataset, by introducing i.i.d. \emph{Rademacher variables} $\sigma_i$ independent of both $T$ and $T'$:
\[
E[G_n] \leq E \left[ 
\sup_{h \in H} \frac{1}{n} \sum_{i=1}^n \sigma_i \big [ l(X'_i,h) - l(X_i, h) \big ]
\right ] 
\]
(since the difference in the square bracket is symmetric, multiplying by the $\sigma$s does not change its distribution!)

\subsubsection{Rademacher's complexity}
\label{sec:rademacher}

We can then define the \emph{Rademacher complexity} \cite{bartlett2005local} as:
\[
R_n (\mathcal{F}) \doteq E \left [ \sup_{f \in \mathcal{F}} \frac{1}{n} \sum_{i=1}^n \sigma_i f(Z_i)
\right ],
\]
where $Z_i \sim p^*$, $\sigma_i \sim \mathcal{U}(-1,+1)$, for any function $f: X \times Y \rightarrow \mathbb{R}$ (for instance, the loss function above).

In opposition, the ``empirical" Rademacher complexity
\[
\hat{R}_n (\mathcal{F}) \doteq E \left [ \sup_{f \in \mathcal{F}} \frac{1}{n} \sum_{i=1}^n \sigma_i f(Z_i)
|
Z_{1:n}
\right ]
\]
is a random variable on $Z_{1:n} = T$, such that
\[
R_n = E[\hat{R}_n].
\]

Putting together the tail bound (\ref{eq:tail-bound}) with the notion of Rademacher complexity we get the desired bound for infinite model spaces (\cite{liang}, Theorem 9). PAC bound holds with
\[
\epsilon = 4 R_n(\mathcal{A}) + \sqrt{\frac{2 \log (2/\delta)}{n}}
\]
where $\mathcal{A}$ is the \emph{loss class}
\[
\mathcal{A} = \{ (x,y) \mapsto l(x,y,h) : h \in \mathcal{H} \}.
\]

Our objective is to derive similar generalisation bounds under more relaxed assumption modelling the epistemic uncertainty about the data distribution. In this work, in particular, we provide an initial study on how realisability can be generalised to a credal setting.

\section{Bounds for realisable finite hypothesis classes} \label{sec:proof}

Under the assumption that: (1) the model space $\mathcal{H}$ is finite, and (2) there exists a hypothesis $h^* \in \mathcal{H}$ that obtains zero expected risk, that is (\ref{eq:realisability}) \emph{realisability} holds, the following result holds\footnote{{\url{https://web.stanford.edu/class/cs229t/notes.pdf}}}. 
\begin{theorem} \label{the:the-4}
Let $\mathcal{H}$ be a hypothesis class, where each hypothesis $h \in \mathcal{H}$ maps some $\mathcal{X}$ to $\mathcal{Y}$, $l$ be the zero-one loss: $l((x, y), h) = \mathbb{I}[y \neq h(x)]$,
$p^*$ be any distribution over $\mathcal{X} \times \mathcal{Y}$ and $\hat{h}$ be the empirical risk minimiser (\ref{eq:empirical-risk-minimiser}). Assume that (1) and (2) hold.
Then, with probability at least $1-\delta$:
\begin{equation} \label{eq:th4-1}
L(\hat{h}) \leq \frac{\log |\mathcal{H}| + \log (1/\delta)}{n}.
\end{equation}
\end{theorem}

\begin{proof}
Let $B = \{h \in \mathcal{H} : L(h) > \epsilon \}$ be the set of `bad' hypotheses. We wish to upper bound the probability \footnote{Note that this is a probability measure on the space of models $\mathcal{H}$, completely distinct
from the data-generating probability $p^*$.} of selecting a bad hypothesis:
\begin{equation} \label{UB_prob}
    P[L(\hat{h}) > \epsilon] = P[\hat{h} \in B].
\end{equation}
Recall that the empirical risk of the ERM is always zero, $\hat{L}(\hat{h}) = 0$, since at least $\hat{L}(h^*) = L(h^*) = 0$. So if we selected a bad hypothesis $(\hat{h} \in B)$, then some
bad hypotheses must have zero empirical risk: $\hat{h} \in B$ implies $\exists h \in B : \hat{L}(h) = 0$ which is equivalent to say that:
\begin{equation} \label{thrm_1_eq6}
    P[\hat{h} \in B] \leq P[\exists h \in B: \hat{L}(h)=0].
\end{equation}
Firstly, we need to bound $P[\hat{L}(h) = 0]$ for any fixed $h \in B$. On each training example, hypothesis $h$ does not err with probability $1 - L(h)$. Since the training examples are i.i.d. and the fact that $L(h) > \epsilon$ for $h \in B$:
\begin{equation*}
    P[\hat{L}(h) = 0] = (1 - L(h))^n \leq (1 - \epsilon)^n \leq e^{-\epsilon n}~~~~ \forall h \in B,
\end{equation*}
where the first equality comes from the examples being i.i.d. (and the shape of the empirical risk), and the last step
follows since $1 - a \leq e^{-a}$. Note that the probability $P[\hat{L}(h) = 0]$ decreases exponentially with $n$.

Secondly, we need to show that the above bound holds simultaneously for all $h \in B$. Recall that 
\[
P(A_1 \cup ... \cup A_K) \leq 
\sum_k P(A_k). 
\]
Note also that the hypothesis space is finite here. 

When applied to the
(non-disjoint) events $A_{h} = \{\hat{L}(h) = 0\}$ the union bound yields:
\begin{equation*}
    P[\exists h \in B : \hat{L}(h) = 0] \leq \sum_{h\in B} P[\hat{L}(h)=0].
\end{equation*}
This is a property of probability measures: the probability of a union of events is smaller than the sum of the probabilities of the individual events. Finally: 
\[
\begin{array}{c}
P[\hat{L}(h) \geq \epsilon] = P[\hat{h} \in B] \leq P[\exists h \in B: \hat{L}(h) = 0] \leq \sum_{h \in B} P[\hat{L}(h) = 0] \\ \leq |B|e^{- \epsilon n}
\leq |\mathcal{H}|e^{- \epsilon n} \doteq \delta, 
\end{array}
\]
where the last inequality holds because $B$ is a subset of $\mathcal{H}$.

By rearranging the above inequality, we obtain: 
\[
\epsilon = \frac{\log |\mathcal{H}| + \log (1/\delta)}{n}, 
\]
i.e., \eqref{eq:th4-1}. 

Inequality \eqref{eq:th4-1} should be interpreted in the following way: with probability at least $1 - \delta$ the expected loss of the empirical risk minimiser $\hat{h}$ (the particular
model selected after training on $n$ examples) is bounded by the ratio on the right-hand side. If we need to obtain the expected risk at most $\epsilon$ with confidence at least $1 - \delta$ we need to train the model on at least 
\[
n = \frac{\log |\mathcal{H}| + \log (1/\delta)}{\epsilon}
\]
many examples.

The result is \emph{distribution-free}  , as it is independent of the choice of $p^{*}(x, y)$,
but does rely on the assumption that training and test distribution are the same.
\end{proof}

\section{Bounds under credal generalisation: a sketch} \label{sec:sketch}

A credal generalisation of Theorem \ref{the:the-4} would thus read as follows.

\begin{theorem} \label{theorem_2}
Let $\mathcal{H}$ be a finite hypothesis class, $\mathcal{P}$ be a credal set (convex set of probability measures) over $\mathcal{X} \times \mathcal{Y}$,
and let a training set of samples $\mathcal{D} = \{(x_i, y_i), i=1,...,n\}$ be drawn from one of the distributions $p^* \in \mathcal{P}$ in this credal set. Let
\begin{equation*}
   \hat{h} \doteq \arg \min_{h \in \mathcal{H}} \hat{L}(h) =
   \arg \min_{h \in \mathcal{H}} \frac{1}{n} \sum^{n}_{i=1} l((x_i, y_i), h) 
\end{equation*}
be the empirical risk minimiser, where $l$ is the $0-1$ loss. Assume that
\begin{equation} \label{eq_thrm_2}
    \exists h^* \in \mathcal{H}, p^* \in \mathcal{P}: \mathbb{E}_{p^*}[l] = L_{p^{*}}(h^*)=0
\end{equation}
\emph{(credal realisability)} holds. Then, with probability at least $1 - \delta$:
\begin{equation}
    P \left[\max_{p \in P} L_{p}(\hat{h}) > \epsilon \right] \leq \epsilon(\mathcal{H}, \mathcal{P}, \delta)
\end{equation}
with $\epsilon$ a function of the size of the model space $\mathcal{H}$, of the credal set $\mathcal{P}$, and of the $\delta$. 
\end{theorem}
\begin{proof}
How does the proof of Theorem \ref{the:the-4} generalise to the credal case? We can note that, if we define 
\begin{equation*}
  B \doteq \left\{h : \max_{p \in P} L_{p}(h) > \epsilon \right\}  
\end{equation*}
as the set of models which do not guarantee the bound in the worst case, we get,
just as in the classical result \eqref{UB_prob}:
\begin{equation*}
 P \left[\max_{p \in P} L_{p}(\hat{h}) > \epsilon \right] = P[\hat{h} \in B].   
\end{equation*}
However, while in the classical case $\hat{L}(\hat{h}) = 0$, in the credal one we only have that, for all $\hat{p} \in \mathcal{P}$: 
\begin{equation*}
    \hat{L}(h^*)= L_{\hat{p}}(h^*) \geq \min_{p \in \mathcal{P}} L_{p}(h^*) = 0,
\end{equation*}
where the last passage follows from the credal realisability assumption \eqref{eq_thrm_2}. Thus, \eqref{thrm_1_eq6} does not hold, and the above argument which tries to bound $P[\hat{L}(h) = 0]$ does not apply here.

One option is to assume that $\hat{p} = p^* = \arg \min_{p} L_{p}(h^*)$. This would allow \eqref{thrm_1_eq6} to remain valid. A different generalisation of realisability can also be proposed:
\begin{equation}
    \forall p \in \mathcal{P}, \exists h_{p}^{*} \in \mathcal{H} : \mathbb{E}_{p}[l(h_{p}^*)] = L_{p}(h_{p}^*)=0, 
\end{equation}
which we can term \textit{uniform credal realisability}. Under the latter assumption, we
would get:
\begin{equation*}
    \hat{L}(\hat{h}) = \min_{h \in \mathcal{H}} \hat{L}(h) = \min_{h \in \mathcal{H}} \hat{L}_{\hat{p}}(h) = L_{\hat{p}}(h_{p}^*)=0. 
\end{equation*}
\end{proof}

\section{Discussion and conclusions} \label{sec:conclusions}

In this work we probed the possible generalisation of statistical learning theory under the assumption that data distributions live in a credal set. We considered the finite realisable case and proposed some generalised realisability conditions to be met to be able to provide generalisation bounds.
It turns out that generalising realisability to the credal case is not trivial, and more than one alternatives should be considered. This will be the subject of our future work.

In particular, this analysis needs to be completed and extended to the case of infinite, non-realisable model spaces. Random sets \cite{molchanov2005theory} should possibly be considered, rather than credal sets, for the structure they provide. At any rate, based on the structure of traditional PAC proofs, the study of generalised concentration inequalities \cite{cozman2008concentration} for credal or random sets appears to be crucial to achieve similar ``distribution-free'' results.

Relevantly, in \cite{caprio2024credal} a different credal approach is proposed which relies on defining a new learning setting in which models are inferred from a (finite) sample of training sets, rather than a single training set, each assumed to have been generated by a single data distribution (as in classical SLT). Within such a setting, the authors are able to derive  generalisation bounds to the expected risk of a model learned in this new learning setting, under the assumption that the epistemic uncertainty induced by the available training sets can be described by a
credal set. 

Last but not least, the presented results still assume a probabilistic PAC framework for generalisation bounds. It may be argued that PAC itself could and should be generalised to epistemic uncertainty, moving the application of the epistemic principle to an even more abstract level.

\section*{Acknowledgements}

This work has received funding from the European Union's Horizon 2020 research and innovation programme under grant agreement No. 964505 (E-pi). 

\bibliographystyle{splncs04}
\bibliography{references}

\end{document}